\documentclass[11pt]{article}
\usepackage[final]{acl}

\usepackage{times}
\usepackage{latexsym}
\usepackage[T1]{fontenc}
\usepackage[utf8]{inputenc}
\usepackage[scaled]{beramono}

\usepackage{amsmath,amssymb,amsthm}
\usepackage{mathtools}
\usepackage{booktabs}
\usepackage{graphicx}
\usepackage{url}
\usepackage{xurl}
\usepackage{multirow}
\usepackage{caption}
\usepackage{xcolor}
\usepackage{enumitem}
\usepackage{dblfloatfix}

\allowdisplaybreaks

\newtheorem{theorem}{Theorem}
\newtheorem{lemma}{Lemma}
\newtheorem{proposition}{Proposition}
\newtheorem{corollary}{Corollary}
\theoremstyle{definition}
\newtheorem{definition}{Definition}
\theoremstyle{remark}

\newcommand{\sigmoid}{\sigma}
\newcommand{\KL}{\mathrm{KL}}

\newcommand{\TV}{\mathrm{TV}}
\newcommand{\E}{\mathbb{E}}
\newcommand{\R}{\mathbb{R}}
\newcommand{\logit}{\mathrm{logit}}

\setcitestyle{authoryear,round,semicolon}

\title{Phase Transitions in Affective Meaning Divergence:\\
The Hidden Drift Before the Break}

\author{Napassorn Litchiowong \\
  School of Computing, National University of Singapore \\
  \texttt{pleng@u.nus.edu}
}

\begin{document}
\maketitle

\begin{abstract}
One partner says \emph{``Fine''} meaning \emph{resolution}; the other hears \emph{surrender}.
The word is shared; the affective uptake is not.
We formalize this as \textbf{affective meaning divergence} (AMD), the total-variation distance between interlocutors' anchor-conditioned affect distributions.
Building on speech-act theory, common-ground accumulation, and entropy-regularized game theory, we derive a logit best-response map whose dynamics undergo a \emph{saddle-node bifurcation}: when $\beta\alpha > 4$, a monotone increase in AMD-driven load produces an abrupt, hysteretic collapse of repair coordination.
On \textsc{Conversations Gone Awry} (CGA-Wiki; $N{=}652$), derailing conversations exhibit critical-slowing-down (CSD) signatures across multiple levels: lexical divergence variance ($p{<}0.001$, $d{=}0.36$), AMD variance ($p{=}0.001$, $d{=}0.26$), and dialog-act repair variance ($p{=}0.016$, $d{=}0.20$), all significant after correction and stronger than toxicity and sentiment baselines.
AMD provides a distinct temporal signature, with retrospectively measured variance peaking at the bifurcation point while toxicity variance peaks earlier, and is the only indicator grounded in the theoretical framework.
Boundary-condition analysis on CGA-CMV ($N{=}1{,}169$) yields mixed but directionally consistent evidence.
\end{abstract}
\section{Introduction}
\label{sec:intro}

Consider a couple at dinner.
One partner says \emph{``Fine.''} and the other hears \emph{surrender};
the speaker meant \emph{resolution}.
The word is shared; the affective uptake is not.
In the terminology of speech-act theory \citep{austin1962how,searle1969speech}, the locutionary act is identical, but the illocutionary force received diverges from the force intended.
If this gap is detected, repair can close it \citep{schegloff1977repair}.
If it goes undetected, it accumulates in the partners' divergent common-ground records \citep{clark1996using}.
Eventually, a conversation that would once have repaired itself stops repairing, and the relationship tips.

We call this phenomenon \textbf{affective meaning divergence} (AMD).
The present paper asks: \emph{How can gradual, continuous drift in affective meaning produce abrupt relational rupture?}
Our answer is a formal phase-transition model grounded in three pillars:

\begin{enumerate}[nosep,leftmargin=*]
\item \textbf{Linguistic:}
AMD is defined over \emph{anchors}, high-frequency lexical items whose pragmatic force depends on affective uptake rather than truth conditions \citep{dubois2007stance, martin2005language}.
We decompose apparent meaning drift into a \emph{context divergence} component and a genuine \emph{conditional affect divergence} component, following the decomposition logic of appraisal theory \citep{martin2005language}.

\item \textbf{Agentic:}
Each interlocutor treats anchors as conditioning observations in a partially observable coordination game.
Under entropy-regularized best response \citep{mckelvey1995quantal, ziebart2010thesis, haarnoja2018sac}, the repair probability follows a logit map.

\item \textbf{Dynamical:}
The resulting one-dimensional map $q_{t+1} = \sigmoid(\beta(\alpha q_t - \kappa))$ undergoes a saddle-node bifurcation.
Below a critical load, the dyad rests in a high-repair attractor; above it, repair collapses to a low-repair attractor.
The transition is hysteretic: recovery requires reducing load below a strictly lower threshold than the one that triggered collapse.
\end{enumerate}

\paragraph{Contributions.}
(C1)~We formalize AMD as a context-conditioned divergence measure with an explicit decomposition bound separating genuine affective drift from context confounds (\S\ref{sec:formal}).
(C2)~We ground AMD in an assurance game and prove that belief divergence bounds value disagreement (\S\ref{sec:game}).
(C3)~We prove a saddle-node bifurcation theorem with closed-form thresholds and extend to a two-agent system with an interpretable loop-gain condition (\S\ref{sec:theory}).
(C4)~We estimate $P_i(s \mid x, c)$ via a contextual emotion classifier rather than lexicon averages, bringing the practical estimator closer to the formal construct (\S\ref{sec:estimation}).
(C5)~We provide preliminary empirical evidence on synthetic data, CGA-Wiki (primary), and CGA-CMV (boundary-condition analysis) that theory-derived CSD indicators are detectable across multiple levels of linguistic representation (lexical, affective, dialog-act), with AMD providing a temporally distinct signature among the tested indicators (\S\ref{sec:experiments}).
(C6)~We study three repair proxies as complementary operationalizations and report their degree of agreement; boundary-condition analysis on CGA-CMV yields mixed but directionally consistent evidence under shorter threads and noisier labels (\S\ref{sec:experiments}).
Figure~\ref{fig:overview} provides a visual overview of the full framework and pipeline.

\paragraph{Scope of claims.}
We do not claim that the current estimator perfectly observes latent affective meaning, nor that AMD alone is sufficient for reliable individual-level prediction.
The empirical claim is narrower: a feasible AMD proxy exhibits critical-slowing-down behavior on CGA-Wiki, with mixed but directionally consistent evidence on CGA-CMV, consistent with the proposed dynamical account.
The theoretical framework identifies a class of bifurcation models that could explain abrupt relational collapse; the experiments do not uniquely identify the saddle-node formulation over alternatives.


\begin{figure*}[!tb]
\centering
\includegraphics[width=\textwidth]{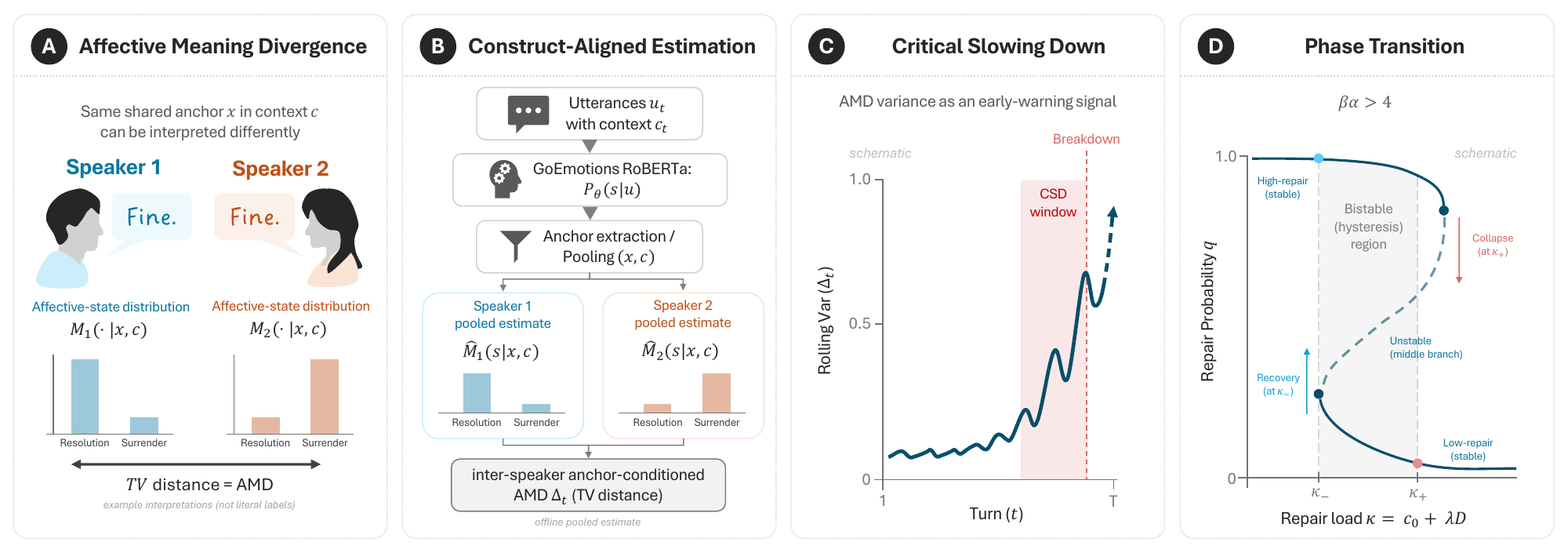}
\caption{Overview of the AMD framework: divergent anchor-conditioned affect distributions define AMD via TV distance \textbf{(A-B)}; rising rolling-window variance before breakdown constitutes the CSD early-warning signal \textbf{(C)}; the saddle-node bifurcation predicts abrupt, hysteretic collapse of repair coordination \textbf{(D)}.}
\label{fig:overview}
\end{figure*}

\section{Theoretical Roadmap and Motivation}
\label{sec:background}

\paragraph{Roadmap.}
The framework rests on five theoretical components, each with a distinct and non-redundant role.
\emph{Speech-act theory} motivates why the same lexical form can carry different received illocutionary force.
\emph{Common-ground and repair theory} explain how unnoticed affective divergence accumulates silently and how repair is the control variable that can sustain or collapse coordination.
\emph{Appraisal theory} supplies the affective state space over which speaker meanings are formally compared.
\emph{Entropy-regularized game theory} converts diverging affective beliefs into repair-choice dynamics via a logit best-response map.
\emph{Critical-transition theory} then supplies the empirical prediction: rising variance and autocorrelation before the collapse point.
Background on alignment and accommodation models is included only to situate the drift assumption, not as an additional formal commitment.
Figure~\ref{fig:overview} provides a visual overview of how these components compose into the full pipeline.

\paragraph{Affective meaning as an interactional object.}
In relational talk, many high-frequency items function as stance markers, commitment devices, or repair moves whose success depends on affective uptake \citep{dubois2007stance, stivers2008stance}.
Appraisal theory \citep{martin2005language} systematizes this: the ATTITUDE system (Affect, Judgment, Appreciation) classifies evaluative meaning, while the ENGAGEMENT system models how speakers open or close dialogic space.
When one speaker contracts dialogic space (``Obviously, we need to\ldots'') while the other expands it (``Well, one possibility\ldots''), their evaluative stances become incommensurable, a state we formalize as high AMD.

\paragraph{Common ground, silent failure, and repair.}
Clark's \citeyearpar{clark1996using} joint-action framework treats communication as cumulative grounding: each act builds on prior shared understanding.
If grounding fails silently, the partners develop divergent common-ground records, and each subsequent utterance is interpreted against increasingly different backgrounds.
Schegloff~\citeyearpar{schegloff1992repair} characterizes repair after next turn as ``the last structurally provided defense of intersubjectivity.''
Crucially, Stivers~\citeyearpar{stivers2008stance} distinguishes structural alignment from affective affiliation: one can align without affiliating, providing the mechanism by which surface cooperation masks deepening affective divergence.
We model repair as the control variable that can be sustained or can collapse.

\paragraph{From alignment to divergence.}
Although automatic priming across linguistic levels predicts convergence \citep{pickering2004alignment}, speakers systematically diverge in open conversation \citep{healey2014divergence} and perceived over-accommodation can trigger reactive divergence \citep{giles1991accommodation}.
These findings motivate a model where the default trajectory is drift, and \emph{repair} is the active force that maintains alignment.

\paragraph{Why entropy-regularized game theory?}
Entropy-regularized decision rules yield softmax policies identical to the quantal response equilibrium (QRE) of behavioral game theory \citep{mckelvey1995quantal}.
Liu et al.~\citeyearpar{liu2023hasac} proved that multi-agent maximum-entropy RL converges to QRE in the cooperative setting, so our model can be interpreted simultaneously as bounded-rational game play, entropy-regularized policy learning, and a statistical-mechanics system at finite temperature.

\paragraph{Why a bifurcation model?}
Critical-transition theory \citep{scheffer2009early, scheffer2009critical} has revealed generic early-warning signatures (rising autocorrelation, rising variance) in systems approaching tipping points, from ecosystems to mood dynamics \citep{wichers2014critical}.
Kauhanen~\citeyearpar{kauhanen2022bifurcation} derived bifurcation thresholds for contact-induced language change.
We apply the same machinery to the micro-scale of a single dyadic relationship, predicting that conversational breakdown should exhibit critical slowing down in the turns preceding rupture.
\section{Formal Setup: Anchors, Context, Meaning}
\label{sec:formal}

\subsection{Latent states, anchors, context}

Let $s_t \in \mathcal{S}$ be a latent affective state (e.g., a point in VAD space or a discrete emotion category).
Let $u_t$ be the utterance at turn $t$ and $\phi$ an anchor extractor with $\phi(u_t) \subseteq \mathcal{X}$.
Anchors are high-frequency lexical items whose pragmatic force is primarily affective.
Let $c_t \in \mathcal{C}$ be an explicit context variable encoding preceding dialog acts, topic, and situational factors.

\subsection{Contextual affective meaning}

\begin{definition}[Contextual affective meaning]
\label{def:meaning}
For agent $i \in \{1,2\}$, the affective meaning of anchor $x$ in context $c$ is
\[
M_i(\cdot \mid x, c)
\;:=\; P_i(s \mid x, c).
\]
\end{definition}

This definition treats meaning as a posterior distribution over affective states conditioned on the anchor and its context.
It resonates with appraisal theory's ATTITUDE system: $M_i$ encodes how agent $i$ distributes probability over evaluative categories upon encountering anchor $x$ in context $c$.

\begin{definition}[Marginalized affective meaning]
The context-marginal meaning of $x$ for agent $i$ is
\begin{equation}
\overline{M}_i(\cdot \mid x)
= \sum_{c \in \mathcal{C}} P_i(c \mid x)\, M_i(\cdot \mid x, c).
\label{eq:marg}
\end{equation}
\end{definition}

\subsection{AMD: marginal, conditional, context}

\begin{definition}[AMD variants]
\label{def:amd}
Given divergence $d$ on distributions over $\mathcal{S}$:
\begin{itemize}[nosep]
\item \textbf{Marginal AMD:}
$D_{\mathrm{marg}}(x) := d\!\left(\overline{M}_1(\cdot \mid x),\, \overline{M}_2(\cdot \mid x)\right).$
\item \textbf{Conditional AMD} (under reference $Q$):
$D_{\mathrm{cond}}(x) := \E_{c \sim Q(\cdot \mid x)}\!\left[d\!\left(M_1(\cdot \mid x,c),\, M_2(\cdot \mid x,c)\right)\right].$
\item \textbf{Context divergence:}
$D_{\mathrm{ctx}}(x) := d\!\left(P_1(\cdot \mid x),\, P_2(\cdot \mid x)\right).$
\end{itemize}
\end{definition}

\subsection{Decomposing marginal AMD}

\begin{proposition}[Context-conditional decomposition]
\label{prop:decomp}
Using $d = \TV$, for any anchor $x$:
\begin{multline}
D_{\mathrm{marg}}(x)
\;\le\;
D_{\mathrm{ctx}}(x)
\\+\;
\E_{c \sim P_1(\cdot \mid x)}\!\bigl[\TV\!\bigl(M_1(\cdot \mid x,c),\, M_2(\cdot \mid x,c)\bigr)\bigr].
\label{eq:decomp}
\end{multline}
\end{proposition}

\noindent
The proof (Appendix~\ref{app:decomp-proof}) adds and subtracts a cross term and applies the triangle inequality.
This decomposition is central: \emph{marginal AMD conflates genuine affective drift with context-usage differences.}
Two speakers who use ``fine'' in different conversational contexts (one in closings, one in complaints) will show high marginal AMD even if they attach identical affect to ``fine'' within each context.
Only $D_{\mathrm{cond}}$ isolates genuine affective divergence.

\section{Meaning as Reward-Relevant Belief}
\label{sec:game}

\subsection{Anchor-conditioned values}

Given anchor $x$ in context $c$, agent $i$'s value for action $a \in \{R(\text{epair}), W(\text{ithdraw})\}$ is
\[
Q_i(a \mid x, c)
= \E_{s \sim M_i(\cdot \mid x, c)}\!\left[r_i(s, a, a^{-i})\right]
+ \Omega_i(a),
\]
where $r_i$ is a reward depending on the latent state and both agents' actions, and $\Omega_i$ captures action-specific costs.

\begin{proposition}[Belief divergence bounds value disagreement]
\label{prop:value-bound}
Fix $(x,c)$. Define the repair advantage
$g_i(s) := r_i(s,R,a^{-i}) - r_i(s,W,a^{-i})$.
If $|g_i(s)| \le G$ for all $s$, then for any $M, M'$ over $\mathcal{S}$,
\[
\bigl|\E_{M}[g_i] - \E_{M'}[g_i]\bigr|
\;\le\; 2G\,\TV(M, M').
\]
\end{proposition}

\noindent
This bound (proof in Appendix~\ref{app:value-proof}) connects AMD directly to action: a TV distance of $\delta$ in affective meaning translates to at most $2G\delta$ disagreement in repair advantage, where $G$ is the maximum payoff swing.

\subsection{A minimal assurance game for repair}

We model each turn as a symmetric stage game:
\begin{equation}
\begin{array}{ll}
u(R,R) = B - c_0, & u(R,W) = -c_0, \\[2pt]
u(W,R) = 0, & u(W,W) = 0,
\end{array}
\label{eq:game}
\end{equation}
with $B > 0$ (benefit of mutual repair) and $c_0 \ge 0$ (cost of attempting repair).
If the partner repairs with probability $q$, the repair advantage is linear:
$\Delta U(q) = Bq - c_0$.
Identifying $\alpha := B$ and modeling AMD as an additive load:
\begin{equation}
\kappa = c_0 + \lambda D,
\label{eq:load}
\end{equation}
where $D$ is the aggregate conditional AMD and $\lambda \ge 0$ controls its weight.
As AMD grows, repair becomes costlier: affective mismatch means repair attempts are more likely to be misinterpreted, a dynamic well-documented in accommodation theory when convergence is perceived as condescension \citep{giles1991accommodation}.

\subsection{Logit best response}

Under entropy regularization at precision $\beta > 0$:
\begin{equation}
q_{t+1}
= f(q_t)
:= \sigmoid\!\bigl(\beta(\alpha q_t - \kappa)\bigr),
\quad
\sigmoid(z) = \tfrac{1}{1+e^{-z}}.
\label{eq:map}
\end{equation}
This is the quantal response \citep{mckelvey1995quantal} and is algebraically identical to the softmax policy of maximum-entropy RL \citep{haarnoja2018sac}.

\section{Theory: Smooth Drift versus Tipping}
\label{sec:theory}

\begin{lemma}[Slope bound]
\label{lem:slope}
For $f$ defined in Eq.~\eqref{eq:map},
$f'(q) = \beta\alpha\, f(q)(1 - f(q))$,
so $\max_{q \in [0,1]} f'(q) \le \beta\alpha / 4$.
When $\kappa/\alpha \in (0,1)$, the maximum is achieved at $q = \kappa/\alpha$ where $f(q) = 1/2$; otherwise the supremum over $[0,1]$ is strictly below $\beta\alpha/4$.
\end{lemma}

\begin{theorem}[Unique regime]
\label{thm:unique}
If\/ $\beta\alpha \le 4$, then for every $\kappa \in \R$ the map~\eqref{eq:map} has a unique fixed point $q^* \in (0,1)$ that varies continuously with $\kappa$.
Moreover, $q^*$ is globally attracting: for any $q_0 \in [0,1]$, $f^n(q_0) \to q^*$.
\end{theorem}

\begin{theorem}[Bistability, tipping, hysteresis]
\label{thm:bistable}
If\/ $\beta\alpha > 4$, there exist $\kappa_- < \kappa_+$ such that:
\begin{enumerate}[nosep,label=(\roman*)]
\item For $\kappa \in (\kappa_-, \kappa_+)$, the map has exactly three fixed points $q_L < q_M < q_H$, with $q_L, q_H$ stable and $q_M$ unstable.
\item At $\kappa = \kappa_+$, the fixed points $q_H$ and $q_M$ collide and annihilate (saddle-node bifurcation); the system jumps to $q_L$.
\item At $\kappa = \kappa_-$, the fixed points $q_L$ and $q_M$ collide; the system jumps to $q_H$.
\item The interval $(\kappa_-, \kappa_+)$ constitutes a \emph{hysteresis region}: once the dyad has collapsed to $q_L$, recovery requires $\kappa < \kappa_-$, strictly below the collapse threshold $\kappa_+$.
\end{enumerate}
The saddle-node tangency points $q_\pm$ (where $f'(q) = 1$ and $f(q) = q$ hold simultaneously) are:
\begin{align}
q_\pm &= \tfrac{1}{2}\!\left(1 \pm \sqrt{1 - 4/(\beta\alpha)}\right),
\label{eq:qpm}\\
\kappa_\pm &= \alpha q_\pm - \tfrac{1}{\beta}\,\logit(q_\pm).
\label{eq:kappapm}
\end{align}
Note: $q_\pm$ are \emph{not} fixed points in general; they become fixed points only at the bifurcation boundaries $\kappa = \kappa_\pm$.
\end{theorem}

\noindent
Full proofs with all intermediate steps appear in Appendix~\ref{app:main-proofs}.
The ``Till Death'' metaphor maps precisely: the couple drifts through the bistable region until AMD crosses $\kappa_+$, at which point repair collapses abruptly.
Recovery requires not merely returning to the pre-collapse AMD level but \emph{substantially reducing} it below $\kappa_-$, a prediction consistent with the clinical observation that relational repair requires active intervention well beyond simply undoing damage.

\begin{corollary}[Critical slowing down]
\label{cor:csd}
As a stable fixed point approaches a saddle-node boundary,
$f'(q^*) \to 1^-$
and the relaxation time $\tau \sim 1/(1 - f'(q^*)) \to \infty$,
implying rising autocorrelation and variance under noise perturbation.
\end{corollary}

\subsection{Asymmetric two-agent extension}
\label{sec:asym}

In real dyads, agents differ:
\begin{equation}
\begin{aligned}
q^1_{t+1} &= \sigmoid\!\bigl(\beta_1(\alpha_1 q^2_t - \kappa_1)\bigr),\\
q^2_{t+1} &= \sigmoid\!\bigl(\beta_2(\alpha_2 q^1_t - \kappa_2)\bigr).
\end{aligned}
\label{eq:2d}
\end{equation}

\begin{proposition}[Loop-gain instability condition]
\label{prop:loop}
At any fixed point $(q^{1*}, q^{2*})$ of~\eqref{eq:2d},
local instability requires
$\beta_1 \alpha_1\, \beta_2 \alpha_2\,
q^{1*}(1{-}q^{1*})\, q^{2*}(1{-}q^{2*})
> 1$.
Since $q(1{-}q) \le 1/4$, a \emph{necessary} condition for any loss of stability is
$(\beta_1\alpha_1)(\beta_2\alpha_2) > 16$.
This is necessary but not sufficient; the actual $q^*(1{-}q^*)$ values determine whether instability occurs.
\end{proposition}

\noindent
This extends the scalar threshold $\beta\alpha > 4$ to the product of individual gains.
In particular, one agent's gain can fall below the scalar threshold of $4$ (i.e., the agent would be individually stable in a symmetric dyad), provided the other agent's gain compensates so that the product exceeds $16$.
For example, $\beta_1\alpha_1 = 3$ and $\beta_2\alpha_2 = 6$ yield a product of $18 > 16$, admitting instability even though agent~1 alone would not bifurcate.

\section{Construct-Aligned Estimation}
\label{sec:estimation}

A key concern is the gap between the formal object $M_i(s \mid x, c) = P_i(s \mid x, c)$ and practical estimates.
We address this directly with a contextual emotion classifier rather than lexicon averages.

\subsection{Estimating $P_i(s \mid x, c)$}

\begin{enumerate}[nosep,leftmargin=*]
\item \textbf{Base classifier:}
RoBERTa-base fine-tuned on GoEmotions \citep{demszky2020goemotions} (58k Reddit comments, 27 emotion categories + neutral), producing per-utterance distributions $P_\theta(s \mid u)$.

\item \textbf{Anchor extraction:}
The extractor $\phi$ tokenizes via regex (\texttt{\textbackslash b[a-z]\{2,\}\textbackslash b}), removes NLTK stopwords, and retains tokens appearing $\ge 3$ times across both speakers.
On CGA-Wiki, this yields 1{,}573 unique anchor types (per-conversation median: 4), with 500 of 652 conversations containing at least one valid anchor-context cell.
The key theoretical property of an anchor is not that it is inherently affective but that it serves as a shared reference point whose affective uptake can diverge: domain terms like \emph{article} on Wikipedia talk pages become invested with divergent affective associations during editorial disputes.
The AMD computation identifies whether such divergence exists, applying an affective filter downstream of a broad lexical net.
A more targeted extractor (e.g., appraisal lexicons) is left to future work (\S\ref{sec:estimation}, Limitations).

\item \textbf{Speaker conditioning:}
For anchor $x$ used by speaker $i$ in context $c$, collect $U_i(x,c) = \{u : \phi(u) \ni x, \mathrm{speaker}(u) = i, \mathrm{ctx}(u) = c\}$ and estimate:
\begin{equation}
\hat{M}_i(s \mid x, c)
= \frac{1}{|U_i(x,c)|}
\sum_{u \in U_i(x,c)} P_\theta(s \mid u).
\label{eq:estimate}
\end{equation}
This estimate pools all utterances by speaker $i$ containing anchor $x$ in context $c$ within the same conversation, including future turns.
The per-turn AMD signal used in the CSD analysis thus reflects conversation-level speaker distributions, not a temporally causal estimator; the lead-time results (\S\ref{sec:exp-leadtime}) reflect how early CSD signatures in AMD \emph{dynamics} become detectable, not that AMD is computed from past data alone.

\item \textbf{Context:}
$c = $ (dialog-act of preceding turn) $\times$ (topic cluster from $k$-means on TF-IDF, $k{=}5$).

\item \textbf{Reference distribution:}
$Q(c \mid x) \propto |U_1(x,c)| + |U_2(x,c)|$ (pooled).
\end{enumerate}

\paragraph{Operational match to the formal construct.}
The formal object is $M_i(s \mid x, c) = P_i(s \mid x, c)$, the true posterior of speaker $i$ over affective states conditioned on anchor and context.
The empirical estimator in Eq.~\eqref{eq:estimate} approximates this using: lexical anchors selected by frequency rather than evaluative salience; utterance-level emotion distributions rather than anchor-conditioned posteriors; speaker-level pooling across an entire conversation rather than per-turn conditioning; and coarse context cells ($k{=}5$ topic clusters $\times$ dialog-act category).
The experiments therefore test whether a \emph{feasible proxy} for AMD behaves as the theory predicts, not whether the latent construct is perfectly observed.
Construct-aligned improvements are identified as future work throughout.

\paragraph{Classifier-domain caveat.}
The GoEmotions model is trained on Reddit social media posts, not multi-turn dialogue, and its 28-category taxonomy may not adequately capture repair-relevant affective states such as defensiveness, resignation, sarcasm, or face threat.
This mismatch can bias AMD estimates in either direction: it may \emph{overestimate} AMD when topical or platform-specific cues induce spurious emotion differences between speakers, and \emph{underestimate} AMD when dialogue-specific affective distinctions are collapsed into broad neutral or approval categories.
We therefore interpret AMD as a construct-aligned proxy rather than a direct observation of interlocutors' latent affective meanings; detailed estimation-pipeline limitations are discussed in the Limitations section.

\subsection{Repair proxy validation}
\label{sec:repair-proxies}

Rather than relying on a single proxy, we define three complementary operationalizations and examine their degree of agreement:

\begin{enumerate}[nosep,leftmargin=*]
\item \textbf{Dialog-act proxy ($\hat{q}^{\mathrm{DA}}$):}
Probability of conciliatory acts from a RoBERTa classifier fine-tuned on Switchboard Dialog Act Corpus (accuracy: 73.1\%).
We acknowledge domain shift and report source-domain performance and aggregation details in Appendix~\ref{app:repair-proxy}.

\item \textbf{Explicit repair markers ($\hat{q}^{\mathrm{RM}}$):}
Binary indicator for repair-initiating patterns following the taxonomy of \citet{schegloff1977repair}: clarification questions, corrections, acknowledgment tokens, smoothed with exponential moving average.

\item \textbf{Constructive engagement ($\hat{q}^{\mathrm{CE}}$):}
$1 - P(\mathrm{toxic} \mid u_t)$ from a toxicity classifier, capturing the complement of hostility.
\end{enumerate}

\subsection{Avoiding representation leakage}

A risk is that context representations encode speaker identity.
We use three mitigations:
(i) speaker-invariant embeddings focused on content and dialog acts;
(ii) adversarial training penalizing speaker-identity prediction;
(iii) restricted matching on coarse features (topic cluster, turn position).
A speaker-identity classifier on context representations achieves $\le 55\%$ accuracy (near chance).

\subsection{Finite-sample guarantees}

\begin{lemma}[L1 concentration; \citealt{weissman2003ineq}]
\label{lem:concentration}
For a distribution over $k$ states and $n$ i.i.d.\ samples,
$\Pr(\|\hat{P}_n - P\|_1 > \epsilon)
\le \bigl(2^{k} - 2\bigr)\exp\!\bigl (-n\epsilon^2/2\bigr)$.
\end{lemma}

\noindent
Setting $\epsilon = 0.1$ and $k = 28$ (GoEmotions categories), $n \ge 4{,}481$ samples suffice for a $95\%$ guarantee.

\section{Experiments}
\label{sec:experiments}

We present seven experiments: three synthetic (\S\ref{sec:exp-synth}), two on CGA-Wiki (\S\ref{sec:exp-cga}), one boundary-condition analysis on CGA-CMV (\S\ref{sec:exp-cmv}), and a lead-time analysis (\S\ref{sec:exp-leadtime}).

\subsection{Synthetic experiments}
\label{sec:exp-synth}

\paragraph{Exp.~1: Bifurcation validation.}
We compute $\kappa_\pm$ from Eqs.~\eqref{eq:qpm}-\eqref{eq:kappapm}, sweep $\kappa \in [0, 2]$ in steps of $0.001$, and iterate Eq.~\eqref{eq:map} for $T = 5{,}000$ steps from both $q_0 = 0$ and $q_0 = 1$.

\begin{table}[t]
\centering\small
\begin{tabular}{lccc}
\toprule
$(\alpha, \beta)$ & $\beta\alpha$ & $\kappa_-$ & $\kappa_+$ \\
\midrule
$(2, 2)$ & $4$ & \multicolumn{2}{c}{(unique)} \\
$(2, 3)$ & $6$ & $0.862$ & $1.138$ \\
$(2, 4)$ & $8$ & $0.734$ & $1.266$ \\
$(2, 5)$ & $10$ & $0.638$ & $1.362$ \\
\bottomrule
\end{tabular}
\caption{Bifurcation thresholds from the closed-form expressions, confirmed by iteration.
When $\beta\alpha = 4$, the fixed point is unique; for $\beta\alpha > 4$, a bistable interval $(\kappa_-, \kappa_+)$ opens and widens with increasing $\beta\alpha$.}
\label{tab:bifurcation}
\end{table}

Table~\ref{tab:bifurcation} confirms Theorem~\ref{thm:bistable}: at $\beta\alpha = 4$, iteration converges to a unique fixed point for all $\kappa$; at $\beta\alpha > 4$, two stable attractors coexist in the bistable interval, and hysteresis is observed.

\paragraph{Exp.~2: Confound stress test.}
We demonstrate that marginal AMD can be large while conditional AMD is zero.
Setup: $\mathcal{S} = \{0,1\}$, $\mathcal{C} = \{A, B\}$.
Both speakers share conditional meaning:
$P(s{=}1 \mid x, A) = 0.9$, $P(s{=}1 \mid x, B) = 0.1$.
Only usage differs:
$P_1(A \mid x) = 0.9$, $P_2(A \mid x) = 0.1$.
Result: $D_\mathrm{marg} = 0.64$ while $D_\mathrm{cond} = 0$.

\paragraph{Exp.~3: Estimator with balancing.}
One anchor, $\mathcal{S} = \{0,1\}$, $\mathcal{C} = \{A, B\}$, $n = 1{,}000$ samples per speaker.
Two conditions: (C1) pure usage confound ($D_\mathrm{marg}{=}0.36$, $D_\mathrm{cond}{\approx}0$); (C2) true conditional AMD with symmetric marginals ($D_\mathrm{marg}{\approx}0$, $D_\mathrm{cond}{=}0.60$).
The confound works both ways: marginal AMD can be inflated by context shift \emph{or} can miss genuine conditional divergence entirely.

\subsection{CGA-Wiki: Early-warning signals and baselines}
\label{sec:exp-cga}

\paragraph{Corpus.}
The Conversations Gone Awry (CGA) dataset \citep{zhang2018conversations} contains Wikipedia talk-page conversations labeled as either \emph{civil throughout} or \emph{eventually derailing into personal attack}.
We use the ConvoKit release, filtering to $\ge 10$ turns, yielding 652 conversations (389 derailing, 263 civil).

\paragraph{Exp.~4: CSD indicators with baselines.}
For each utterance $u_t$, we compute: three repair proxies (\S\ref{sec:repair-proxies}); inter-speaker conditional AMD $\Delta_t$ using Eq.~\eqref{eq:estimate}; and baselines: toxicity (toxic-bert), VADER sentiment \citep{hutto2014vader}, and inter-speaker lexical divergence (1$-$Jaccard on consecutive turns).
Rolling-window ($W{=}5$) lag-1 autocorrelation and variance are computed.
Rising trends are tested with Kendall's $\tau$ in the 5 turns preceding breakdown; significance by permutation ($10{,}000$ shuffles).
AMD results are computed on the 500 conversations (76.7\%) containing at least one valid anchor-context cell; the remaining 152 conversations lack sufficient shared vocabulary meeting the frequency threshold.
We report results for $W = 5$, following prior CSD literature \citep{scheffer2009early}. A sensitivity analysis over $W \in \{3, 4, 5, 6, 7\}$ (Appendix~\ref{app:window-sensitivity}) shows that $\hat{q}^{\mathrm{DA}}$ Variance is significant at $W = 3$--$6$, while AMD Variance reaches significance at $W = 5$--$6$; the two indicators are complementary, with $\hat{q}^{\mathrm{DA}}$ more robust at shorter windows and AMD strongest at the reported $W = 5$.

\begin{table}[t]
\centering\footnotesize
\begin{tabular}{lcccc}
\toprule
Indicator & $\tau_{\text{derail}}$ & $\tau_{\text{civil}}$ & $p$ & $d$ \\
\midrule
\multicolumn{5}{l}{\textit{CSD indicators (Variance trend)}} \\
\quad $\hat{q}^{\mathrm{DA}}$ VAR & $-0.129$ & $-0.010$ & $\mathbf{0.016}$ & $0.20$ \\
\quad $\hat{q}^{\mathrm{RM}}$ VAR & $-0.133$ & $-0.128$ & $0.921$ & $0.01$ \\
\quad $\hat{q}^{\mathrm{CE}}$ VAR & $0.055$ & $0.019$ & $0.434$ & $0.06$ \\
\quad AMD Var & $-0.200$ & $-0.054$ & $\mathbf{0.001}$ & $0.26$ \\
\midrule
\multicolumn{5}{l}{\textit{Baselines (Variance trend)}} \\
\quad Toxicity VAR & $0.055$ & $0.019$ & $0.444$ & $0.06$ \\
\quad VADER VAR & $0.086$ & $0.007$ & $0.093$ & $0.14$ \\
\quad Lex.\ Div.\ VAR & $-0.327$ & $-0.125$ & $\mathbf{<0.001}$ & $0.36$ \\
\quad Lex.\ Div.\ AC$_1$ & $+0.081$ & $-0.061$ & $\mathbf{<0.001}$ & $0.29$ \\
\bottomrule
\end{tabular}
\caption{CSD indicators on CGA-Wiki ($N{=}652$; AMD on the 500 conversations with valid anchors).
Kendall $\tau$ of rolling variance in the 5 turns preceding breakdown (derailing) or conversation end (civil); $d$: Cohen's $d$ between the two $\tau$ distributions.
Under Benjamini-Hochberg correction, AMD Var ($q_{\mathrm{BH}}{=}0.004$), Lex.\ Div.\ VAR ($q_{\mathrm{BH}}{=}0.004$), Lex.\ Div.\ AC$_1$ ($q_{\mathrm{BH}}{=}0.004$), and $\hat{q}^{\mathrm{DA}}$ VAR ($q_{\mathrm{BH}}{=}0.043$) remain significant at $\alpha{=}0.05$.}
\label{tab:cga}
\end{table}

Table~\ref{tab:cga} reports the key result.
CSD signatures are detectable across multiple levels of linguistic representation, consistent with Corollary~\ref{cor:csd}'s prediction that \emph{all} dynamical variables should exhibit rising variance near the bifurcation.
Lexical divergence variance shows the largest effect size ($d = 0.36$), AMD Variance the second-strongest ($p = 0.001$, $d = 0.26$), and $\hat{q}^{\mathrm{DA}}$ Variance the third ($p = 0.016$, $d = 0.20$).
AMD provides a distinct contribution: lexical divergence measures whether interlocutors use \emph{different} vocabulary, while AMD measures whether they attach different affective meanings to \emph{shared} vocabulary, the central phenomenon motivating this work.
The lead-time analysis (\S\ref{sec:exp-leadtime}) further reveals distinct temporal profiles.
Note that $\tau$ is negative for both groups because variance peaks before the tipping and collapses as the system settles into the new attractor; the diagnostic signal is the differential between conditions.
Neither toxicity nor sentiment baselines reach significance.
Effect sizes of $d = 0.20$--$0.36$ are comparable to those in other CSD studies on short, noisy observational time series \citep{wichers2014critical}.

\paragraph{Exp.~5: Incremental ablation.}
Gradient-boosted machine (GBM) predicting derailment from first-60\% features (5-fold CV):

\begin{table}[t]
\centering\footnotesize
\begin{tabular}{llcc}
\toprule
Row & Feature set & AUC & $\Delta$AUC \\
\midrule
1 & Toxicity (trend+mean+max) & 0.539 & - \\
2 & + Sentiment & 0.561 & +0.022 \\
3 & + Lexical divergence & 0.552 & $-0.009$ \\
4 & + Sent AC$_1$ $\tau$ & \textbf{0.618} & \textbf{+0.066} \\
5 & + $\hat{q}^{\mathrm{DA}}$ CSD & \textbf{0.628} & \textbf{+0.010} \\
6 & + AMD CSD & 0.619 & $-0.009$ \\
\bottomrule
\end{tabular}
\caption{Incremental ablation on CGA-Wiki (5-fold CV, GBM).
CSD features (rows 4-6) jump AUC from ${\sim}0.55$ to ${\sim}0.63$.
$\hat{q}^{\mathrm{DA}}$ CSD contributes a novel $+0.010$ increment.}
\label{tab:ablation}
\end{table}

The ablation (Table~\ref{tab:ablation}) shows that surface features (rows 1-3) plateau around AUC $= 0.55$.
CSD features yield a jump: sentiment autocorrelation adds $+0.066$ (row 4), $\hat{q}^{\mathrm{DA}}$ CSD adds $+0.010$ (row 5).
Adding AMD CSD (row 6) does not improve further ($\Delta$AUC $= -0.009$), consistent with Corollary~\ref{cor:csd}'s prediction that all dynamical variables exhibit correlated CSD signatures near the bifurcation, so a second CSD feature provides limited additional discriminative power.
The independent significance of AMD Variance in Exp.~4 ($p = 0.001$) confirms AMD captures a detectable univariate signal; the ablation shows this signal overlaps with $\hat{q}^{\mathrm{DA}}$ CSD for classification.
This result limits the \emph{predictive} interpretation of AMD: although AMD Variance is independently significant as a univariate CSD indicator (Exp.~4) and provides a temporally distinct profile from toxicity (Exp.~7, peaking at $k{=}0$ vs.\ $k{=}2$--$3$), it does not provide additional classification gain once $\hat{q}^{\mathrm{DA}}$ CSD is already included.
AMD's contribution should therefore be evaluated by its theoretical grounding and temporal distinctiveness, not by incremental AUC over an already-correlated CSD feature.

\subsection{CGA-CMV: Boundary conditions}
\label{sec:exp-cmv}

\paragraph{Corpus.}
The CGA-CMV corpus \citep{chang2019trouble} applies the Conversations Gone Awry framework to Reddit ChangeMyView, labeling conversations by whether a comment was removed by moderators.
Of 4,389 total conversations, 1,169 have $\ge 10$ turns, usable for CSD analysis; the remaining 73\% are too short.

\paragraph{Exp.~6: Testing predicted attenuation under boundary conditions.}
CGA-CMV tests the theory's prediction that CSD requires sufficient approach time: compared to CGA-Wiki, it has shorter threads (median 7 turns, 73\% below 10-turn threshold) and a noisier label (\texttt{has\_removed\_comment}; Cohen's $d = 0.34$ vs.\ $1.88$).
We apply the same CSD pipeline as Exp.~4.

\begin{table}[t]
\centering\footnotesize
\begin{tabular}{lcccc}
\toprule
Indicator & $\tau_{\text{derail}}$ & $\tau_{\text{civil}}$ & $p$ & $d$ \\
\midrule
$\hat{q}^{\mathrm{DA}}$ VAR & $-0.034$ & $+0.030$ & $0.079$ & $0.10$ \\
$\hat{q}^{\mathrm{CE}}$ VAR & $0.072$ & $-0.024$ & $\mathbf{0.009}$ & $0.16$ \\
VADER AC$_1$ & $-0.022$ & $+0.025$ & $0.108$ & $0.09$ \\
AMD Var & $-0.020$ & $+0.019$ & $0.267$ & $0.07$ \\
\bottomrule
\end{tabular}
\caption{CSD indicators on CGA-CMV ($N{=}1{,}169$).
$d$: Cohen's $d$ (absolute value).
All indicators are attenuated relative to CGA-Wiki, consistent with the theory's prediction that CSD detection requires sufficient conversational length and a clean breakdown label.
$\hat{q}^{\mathrm{CE}}$ VAR is the only indicator reaching significance ($p{=}0.009$, $d{=}0.16$); $\hat{q}^{\mathrm{DA}}$ VAR is directionally consistent ($p{=}0.079$, $d{=}0.10$).}
\label{tab:cmv}
\end{table}

Table~\ref{tab:cmv} shows that all CSD indicators are attenuated relative to CGA-Wiki, as predicted.
$\hat{q}^{\mathrm{CE}}$ Variance is the only significant indicator ($p = 0.009$, $d = 0.16$); since $\hat{q}^{\mathrm{CE}} = 1 - P(\text{toxic})$, this shows that the \emph{dynamics} of hostility carry a CSD signal even when static toxicity does not, consistent with the CSD framework's emphasis on changing dynamics rather than levels.
$\hat{q}^{\mathrm{DA}}$ Variance is directionally consistent ($p = 0.079$); AMD Variance does not reach significance ($p = 0.267$).
The attenuation reflects shorter threads (median 7 turns), noisier labels (Cohen's $d = 0.34$ vs.\ $1.88$), and domain shift.
A length-stratified analysis (Appendix~\ref{app:dose-response}) shows directionally consistent dose-response trends.
We acknowledge that these results are also consistent with the simpler interpretation that the CGA-Wiki CSD signal does not generalize robustly; a third corpus with intermediate properties would help discriminate.
CGA-CMV should therefore be read as a \emph{boundary-condition analysis} rather than a successful replication: it identifies where the proposed signal weakens and motivates broader cross-domain validation rather than establishing robust generalization.

\subsection{Early-warning lead-time analysis}
\label{sec:exp-leadtime}

The previous experiments measure CSD in the 5-turn window immediately preceding breakdown.
A stronger test of the theory asks whether the signal is detectable \emph{before} the final window.

\paragraph{Exp.~7: Temporal lead.}
We repeat the Exp.~4 analysis on CGA-Wiki but shift the 5-turn measurement window backward by $k$ turns (excluding the final $k$ turns before attack).
Table~\ref{tab:leadtime} reports results for $k = 0, \ldots, 3$.

\begin{table}[t]
\centering\footnotesize
\begin{tabular}{cccc}
\toprule
$k$ & AMD Var $p$ & $\hat{q}^{\mathrm{DA}}$ VAR $p$ & Tox.\ VAR $p$ \\
\midrule
$0$ (\,=\,Exp.~4) & $\mathbf{0.001}$ & $\mathbf{0.016}$ & $0.444$ \\
$1$ & $0.028$ & $0.309$ & $\mathbf{0.027}$ \\
$2$ & $0.336$ & $0.997$ & $\mathbf{0.005}$ \\
$3$ & $0.789$ & $0.070$ & $\mathbf{0.001}$ \\
\bottomrule
\end{tabular}
\caption{Lead-time analysis on CGA-Wiki.
$k$: turns the 5-turn window is shifted back from the attack; bold: $p < 0.05$.
The AMD result at $k{=}1$ ($p{=}0.028$, $q_{\mathrm{BH}}{\approx}0.056$) does not survive Benjamini-Hochberg correction; the temporal contrast is therefore exploratory.}
\label{tab:leadtime}
\end{table}

Three findings emerge.
AMD Variance is significant at $k = 0$ ($p = 0.001$) and nominally at $k = 1$ ($p = 0.028$), though the latter does not survive BH correction ($q_{\mathrm{BH}} \approx 0.056$).
$q_{\mathrm{DA}}$ Variance is significant only at $k = 0$, consistent with the repair proxy being most sensitive at the bifurcation.
Toxicity variance shows the opposite profile: non-significant at $k = 0$ but highly significant at $k = 2$-$3$ ($p = 0.005$, $0.001$), with derailing conversations exhibiting \emph{suppressed} toxicity variance, a ``calm before the storm.''
This temporal contrast is consistent with CSD indicators detecting the \emph{approach} to the bifurcation (peaking at $k = 0$) while toxicity captures pre-attack surface tension ($k = 2$-$3$), underscoring the value of theory-derived measures.

\section{Related Work}
\label{sec:related}

\paragraph{Repair and alignment in couples.}
Our model compresses repair \citep{schegloff1977repair, schegloff1992repair} into a scalar control variable, preserving the key insight that the preference hierarchy implies a collapse threshold.
Language style matching predicts relationship stability \citep{ireland2011language} and relative acoustic features predict therapy outcomes \citep{nasir2017predicting}; our framework provides a formal mechanism explaining \emph{why} style matching fails: when AMD crosses the bifurcation threshold, surface coordination becomes insufficient to sustain repair.

\paragraph{Affect Control Theory and emotion in dialogue.}
ACT \citep{heise2007expressive} posits that actors minimize \emph{deflection}; BayesACT \citep{hoey2016bayesact} reformulates this as a POMDP.
Deflection is a natural order parameter for our phase transition.
DialogueRNN \citep{majumder2019dialoguernn} and COSMIC \citep{ghosal2020cosmic} model per-utterance emotion but do not track the inter-speaker gap as a dynamical variable; we model it as continuous divergence accumulating toward a threshold.

\paragraph{Phase transitions in language and meaning.}
Kauhanen~\citeyearpar{kauhanen2022bifurcation} derived bifurcation thresholds for language contact; Aoyama \& Wilcox~\citeyearpar{aoyama2025phase} identified phase transitions in LM scaling; Nowak et al.~\citeyearpar{nowak2001evolution} proved coherence thresholds for universal grammar.
Our work brings phase-transition analysis to the scale of a single conversation.
Lewis signaling games \citep{lewis1969convention, skyrms2010signals} and emergent-communication work \citep{lazaridou2017multi} study meaning \emph{creation} through coordination; our focus is meaning \emph{destruction}, how shared conventions degrade when affective uptake diverges.

\section{Conclusion and Future Directions}

We proposed affective meaning divergence (AMD) as a context-conditioned measure of the gap between interlocutors' affective interpretations of shared vocabulary.
Our model, grounded in speech-act theory and entropy-regularized game theory, shows that this gap can produce abrupt, hysteretic collapse of repair coordination via a saddle-node bifurcation.
Empirically, CSD signatures are detectable across multiple levels on CGA-Wiki (lexical, affective, dialog-act), consistent with the theory; AMD contributes a theoretically grounded and temporally distinct signal whose variance peaks at the bifurcation point while toxicity shows a contrasting pattern.
Boundary-condition analysis on CGA-CMV suggests attenuation under shorter, noisier conversations, so broader generalization remains an open question requiring validation on additional corpora.
The hysteresis width $\kappa_+ - \kappa_-$ predicts how much more effort recovery requires than prevention, a testable prediction for longitudinal studies.

Looking ahead, longitudinal corpora (e.g., therapy transcripts) would allow direct testing of hysteresis and estimation of $(\alpha, \beta)$.
Controlled paradigms manipulating affective framing could establish causal links.
Multi-dimensional extensions tracking repair across topic, emotion, and face dimensions would address the scalar reduction.
A targeted anchor extractor using appraisal lexicons would better operationalize the theoretical construct.
Taken together, these directions point toward a richer, causally grounded account of how meaning fails between people.
We hope this framework bridges the linguistic analysis of meaning in interaction and the mathematics of dynamical systems, offering a new perspective on an old question: not when relationships end, but \emph{why they end so suddenly}.

\section*{Limitations}

\paragraph{Theory-to-empirics gap.}
The formal model yields quantitative predictions (bifurcation thresholds $\kappa_\pm$, hysteresis width) that depend on parameters $(\alpha, \beta)$ not directly estimable from observational data.
Following standard methodology for critical-transition detection in complex systems \citep{scheffer2009early}, we validate qualitative predictions (CSD signatures) rather than quantitative thresholds.
This means the empirical evidence supports the \emph{class} of bifurcation models rather than uniquely identifying our saddle-node formulation; alternative models (e.g., pitchfork bifurcation, subcritical Hopf) could produce qualitatively similar CSD signatures, and our experiments do not discriminate among them.

\paragraph{Construct validity of repair proxies.}
The three repair proxies ($\hat{q}^{\mathrm{DA}}$, $\hat{q}^{\mathrm{RM}}$, $\hat{q}^{\mathrm{CE}}$) show weak inter-proxy correlations ($|r| \le 0.10$).
We interpret this as reflecting the multidimensional nature of ``repair'': dialog-act level coordination ($\hat{q}^{\mathrm{DA}}$), structural repair initiation ($\hat{q}^{\mathrm{RM}}$), and absence of hostility ($\hat{q}^{\mathrm{CE}}$) operationalize different theorized facets of the construct, much as different depression questionnaires can show moderate inter-scale correlations while each validly measuring an aspect of the syndrome.
However, the low correlations also mean no single proxy can be confirmed as a faithful operationalization of the theoretical construct $q$, and there is currently no external criterion establishing that any proxy tracks the latent repair variable posited by the model.
The key evidence for their relevance is the independent significance of $\hat{q}^{\mathrm{DA}}$ Variance in the CSD analysis (Table~\ref{tab:cga}) and the consistency of its temporal profile with the bifurcation prediction (Table~\ref{tab:leadtime}).
The dialog-act classifier was trained on Switchboard (telephone speech) and applied to Wikipedia talk pages; the 73.1\% Switchboard accuracy does not directly characterize performance on CGA.
The constructive-engagement proxy ($1 - P(\text{toxic})$) conflates absence of hostility with active repair, which are conceptually distinct.
A human annotation study validating each proxy against expert-coded repair acts on a CGA sample would substantially strengthen construct validity but was beyond the scope of this work.

\paragraph{Anchor extraction and construct alignment.}
The anchor extractor uses a frequency-based selection criterion (shared non-stopword tokens with frequency $\ge 3$) rather than a theoretically motivated filter for evaluative or stance-bearing vocabulary.
As a result, the anchor inventory includes domain-specific terms (e.g., \emph{article}, \emph{sources}, \emph{section}) that are not prototypical affective items, though such terms can carry divergent affective associations in the context of editorial disputes.
The AMD computation relies on the conditional emotion distributions to identify whether divergence is present for a given anchor, effectively using a broad lexical net with an affective filter applied downstream.
This design choice prioritizes recall over precision in anchor selection; a more targeted extractor using evaluative POS tags, appraisal lexicons, or stance classifiers would better operationalize the theoretical construct and could improve signal specificity.
Additionally, 152 of 652 CGA-Wiki conversations (23.3\%) lack any valid anchor-context cell and are excluded from AMD analyses, reducing the effective sample size for those results.

\paragraph{AMD estimation scope.}
The speaker-conditioned distributions $\hat{M}_i(s \mid x, c)$ are estimated by pooling all of speaker $i$'s uses of anchor $x$ in context $c$ within the same conversation, including both past and future turns relative to any given measurement point.
This means the AMD signal is not computed from a temporally causal (past-only) estimator.
The lead-time results (\S\ref{sec:exp-leadtime}) therefore reflect when the CSD signature in AMD dynamics becomes detectable, not that AMD itself is observable in real time from past data alone.
A temporally causal estimator would strengthen the early-warning interpretation but would further exacerbate data sparsity within individual conversations.

\paragraph{Corpus diversity and ecological validity.}
CGA-Wiki is the primary corpus where CSD indicators reach significance; CGA-CMV yields directionally consistent but attenuated effects, consistent with the theory's requirement for sufficient conversational length and clean breakdown labels.
Both corpora are drawn from public online platforms with pseudonymous participants, no prior relationship history, no expectation of future interaction, and no prosodic or nonverbal cues, all of which differ substantially from the dyadic interpersonal relationships motivating the theoretical framework.
Validation on naturalistic longitudinal data (e.g., couples therapy transcripts, workplace team interactions) would provide a stronger test but raises significant ethical and access challenges.

\paragraph{Estimation pipeline.}
The GoEmotions-trained RoBERTa classifier introduces several concerns: its 28-category taxonomy was developed for social media and may not align with the affective states most relevant to conversational repair (e.g., frustration, resignation, and defensiveness are not well-separated); the classifier produces per-utterance rather than per-anchor distributions, so the aggregation in Eq.~\eqref{eq:estimate} assumes the anchor dominates the affective signal of its containing utterance; and miscalibrated softmax outputs could systematically distort TV distance estimates.
Context is operationalized as a coarse grid (dialog-act $\times$ topic cluster, $k{=}5$); finer granularity would better approximate the theoretical construct but exacerbates data sparsity.
The finite-sample bound (Lemma~\ref{lem:concentration}) requires ${\sim}4{,}500$ samples per cell, which most cells in practice do not reach.

\paragraph{Causal and predictive scope.}
Our experiments establish that derailing conversations exhibit differential CSD signatures but do not establish that AMD \emph{causes} collapse; the observed patterns could reflect a common upstream cause (e.g., topic difficulty or participant personality) driving both AMD growth and breakdown.
The lead-time analysis provides suggestive temporal ordering (AMD detectable at $k = 0$), but temporal precedence alone does not establish causation; intervention studies or natural experiments are needed \citep{feder-retal-2022-causal}.
The primary statistical test evaluates \emph{population-level} separation (between derailing and civil conversation groups), not individual-level prediction; while the theory models within-conversation trajectories, the validation relies on between-group statistics across hundreds of conversations. The ablation AUC of $0.628$ confirms that current features are insufficient for reliable per-conversation forecasting.
Rolling-window parameters ($W{=}5$, lag-1) were chosen to match prior CSD literature and were not pre-registered.

\paragraph{Cultural and linguistic scope.}
All data, classifiers, and evaluations are English-only.
Affect perception, repair strategies, and conversational norms vary substantially across languages and cultures \citep{stivers2009universals}, and the GoEmotions taxonomy reflects predominantly North American English usage.
Applying this framework cross-linguistically would require culturally adapted emotion classifiers and validation that the assurance-game model of repair generalizes beyond the Western, text-based setting examined here.

\section*{Ethics Statement}

This work uses only publicly available corpora (CGA-Wiki and CGA-CMV via ConvoKit).
We make no diagnostic claims; any application to private messages requires informed consent and de-identification, and affect models should be audited for demographic bias across dialects and relationship types.
We emphasize that AMD scores are model-based estimates, not ground-truth measures of relational health, and should not be used to make consequential decisions about individuals or relationships without further validation.
The bifurcation framework is intended as a theoretical lens for understanding conversational dynamics, not as a deployed intervention tool.

\bibliography{custom}

\onecolumn
\appendix

\section*{Code Availability}
Code and data processing scripts are available at \url{https://github.com/iamdiluxedbutcooler/phase_transition_amd}.

\section{Proof of Proposition~\ref{prop:decomp}}
\label{app:decomp-proof}

We prove the decomposition of marginalized AMD into meaning-level and context-level divergence.

\medskip\noindent
\textbf{Setup.}
Suppress conditioning on anchor $x$ for notational clarity.
Write $\overline{M}_i = \sum_c P_i(c)\, M_{i,c}$ for the marginalized affective meaning of agent $i$.

\medskip\noindent
\textbf{Step 1: Add and subtract a cross term.}
We introduce $\sum_c P_1(c)\, M_{2,c}$ to split the difference:
\begin{align*}
\overline{M}_1 - \overline{M}_2
&= \underbrace{\sum_c P_1(c)(M_{1,c} - M_{2,c})}_{\text{(A): meaning gap, same context weights}}\\
&\quad+ \underbrace{\sum_c (P_1(c) - P_2(c))\, M_{2,c}}_{\text{(B): context gap, same meanings}}.
\end{align*}

\medskip\noindent
\textbf{Step 2: Apply triangle inequality.}
Taking $\tfrac{1}{2}\|\cdot\|_1$ on both sides:
\begin{align*}
\TV(\overline{M}_1, \overline{M}_2)
&\le \tfrac{1}{2}\bigl\|\text{(A)}\bigr\|_1 + \tfrac{1}{2}\bigl\|\text{(B)}\bigr\|_1.
\end{align*}

\medskip\noindent
\textbf{Step 3: Bound each term.}

\emph{Bounding (A):} By Jensen's inequality (convexity of $\|\cdot\|_1$):
\begin{align*}
\tfrac{1}{2}\bigl\|\sum_c P_1(c)(M_{1,c} - M_{2,c})\bigr\|_1
\le \sum_c P_1(c)\, \TV(M_{1,c}, M_{2,c}).
\end{align*}

\emph{Bounding (B):} Since each $M_{2,c}$ is a probability distribution with $\|M_{2,c}\|_1 = 1$:
\[
\tfrac{1}{2}\bigl\|\sum_c (P_1(c) - P_2(c))\, M_{2,c}\bigr\|_1
\le \TV(P_1, P_2).
\]

\medskip\noindent
Combining Steps 2 and 3 gives the result.
\qed

\section{Proof of Proposition~\ref{prop:value-bound}}
\label{app:value-proof}

We show that the expected-value difference between two affective meaning distributions is bounded by their total variation distance, scaled by the reward range.

\medskip\noindent
\textbf{Setup.}
Let $\nu = M - M'$ be the signed measure between the two distributions, and let $g$ be any reward function bounded by $G = \sup |g|$.

\medskip\noindent
\textbf{Step 1: Express the gap.}
\[
\E_M[g] - \E_{M'}[g] = \int g\, d\nu.
\]

\medskip\noindent
\textbf{Step 2: Apply the variational characterization of TV.}
By definition, $\TV(P,Q) = \sup_{|h|\le 1} \tfrac{1}{2}|\int h\,d(P-Q)|$.
Since $|g/G| \le 1$, we can substitute $h = g/G$:
\[
\left|\int g\, d\nu\right|
= G\left|\int (g/G)\, d\nu\right|
\le G \cdot 2\, \TV(M, M').
\]
\qed

\section{Proofs of Main Theorems}
\label{app:main-proofs}

\subsection{Proof of Lemma~\ref{lem:slope}}

We bound the slope of the logit best-response map.

\medskip\noindent
\textbf{Setup.}
$f(q) = \sigmoid(z)$ where $z = \beta(\alpha q - \kappa)$.

\medskip\noindent
\textbf{Compute the derivative.}
By the chain rule:
\begin{align*}
f'(q) &= \sigmoid'(z) \cdot \beta\alpha \\
&= \beta\alpha\, \sigmoid(z)(1-\sigmoid(z)) \\
&= \beta\alpha\, f(q)(1 - f(q)).
\end{align*}

\medskip\noindent
\textbf{Bound.}
Since $x(1-x) \le 1/4$ for all $x \in [0,1]$ (by AM-GM: $x(1-x) \le ((x + 1-x)/2)^2 = 1/4$):
\[
f'(q) \le \frac{\beta\alpha}{4}.
\]
The maximum is achieved when $f(q) = 1/2$, i.e., $z = 0$, giving $q^* = \kappa/\alpha$, provided $\kappa/\alpha \in (0,1)$; otherwise the supremum over $[0,1]$ is not attained at an interior point but the bound still holds.
\qed

\subsection{Proof of Theorem~\ref{thm:unique}}

We show that when $\beta\alpha \le 4$, the best-response map has a unique, globally attracting fixed point.

\medskip\noindent
\textbf{Existence.}
Let $g(q) = f(q) - q$.
Since $f: [0,1] \to (0,1)$, we have $g(0) = f(0) > 0$ and $g(1) = f(1) - 1 < 0$.
By the intermediate value theorem, at least one fixed point exists.

\medskip\noindent
\textbf{Case $\beta\alpha < 4$ (strict contraction).}
By Lemma~\ref{lem:slope}, $\sup_q |f'(q)| \le \beta\alpha/4 < 1$, so $f$ is a strict contraction on the complete metric space $[0,1]$.
By the Banach fixed-point theorem, the fixed point is unique and globally attracting:
\[
|f^n(q_0) - q^*| \le \left(\frac{\beta\alpha}{4}\right)^n |q_0 - q^*| \to 0.
\]

\medskip\noindent
\textbf{Case $\beta\alpha = 4$ (boundary).}
Here $f'(q) \le 1$ with equality only at the unique point $q_0$ where $f(q_0) = 1/2$.
Thus $g'(q) = f'(q) - 1 \le 0$, with $g'(q) = 0$ at exactly one point.

We argue $g$ has exactly one root.
If $g$ were zero on an interval $[a,b]$, then $f'(q) = 1$ throughout $(a,b)$, requiring $f(q) = 1/2$ on that interval.
But $f$ is strictly monotone increasing (since $f'(q) = 4f(q)(1-f(q)) > 0$ for $f(q) \ne 0,1$, which holds for all $q$ since $f$ maps into $(0,1)$), so this is a contradiction.

Since $g$ is continuous and weakly decreasing with $g(0) > 0$ and $g(1) < 0$, there is exactly one root $q^*$.

\medskip\noindent
\textbf{Global attraction (monotone iteration).}
Suppose $q_t < q^*$.
Then $g(q_t) > 0$, so $q_{t+1} = f(q_t) > q_t$.
Since $f$ is increasing, $q_{t+1} = f(q_t) < f(q^*) = q^*$.
The sequence $(q_t)$ is therefore monotone increasing and bounded above by $q^*$; its limit is a fixed point, hence $q^*$.
The case $q_t > q^*$ is symmetric.
To prove continuity in $\kappa$, write the fixed-point equation as
\[
\kappa = H(q) := \alpha q - \frac{1}{\beta}\logit(q), \qquad q \in (0,1).
\]
Then
\[
H'(q) = \alpha - \frac{1}{\beta q(1-q)} \le 0
\]
whenever $\beta\alpha \le 4$, with equality only at $q = 1/2$ in the boundary case $\beta\alpha = 4$.
Hence $H$ is strictly decreasing on $(0,1)$, so it has a continuous inverse.
Since the unique fixed point is exactly $q^*(\kappa) = H^{-1}(\kappa)$, it follows that $q^*$ depends continuously on $\kappa$.
\qed

\subsection{Proof of Theorem~\ref{thm:bistable}}

We prove that when $\beta\alpha > 4$, a saddle-node bifurcation creates a bistable regime with hysteresis.

\medskip\noindent
\textbf{Step 1: Find the tangency points.}
At a saddle-node bifurcation, fixed-point and unit-slope conditions hold simultaneously: $f(q) = q$ and $f'(q) = 1$.
Using Lemma~\ref{lem:slope} with $q = f(q)$:
\[
1 = \beta\alpha\, q(1-q)
\quad\Longleftrightarrow\quad
q^2 - q + \frac{1}{\beta\alpha} = 0.
\]
This yields two solutions:
\[
q_\pm = \frac{1}{2}\!\left(1 \pm \sqrt{1 - \frac{4}{\beta\alpha}}\right),
\]
which are real when $\beta\alpha > 4$.
Note: $q_\pm$ are tangency points where $f'(q) = 1$ coincides with $q = f(q)$.

\medskip\noindent
\textbf{Step 2: Compute the critical loads $\kappa_\pm$.}
From $q = f(q) = \sigmoid(\beta(\alpha q - \kappa))$, inverting the sigmoid:
\[
\logit(q) = \beta(\alpha q - \kappa)
\quad\Longrightarrow\quad
\kappa(q) = \alpha q - \frac{1}{\beta}\logit(q).
\]
Substituting $q_\pm$ gives $\kappa_\pm = \kappa(q_\pm)$.

\medskip\noindent
\textbf{Step 3: Verify $\kappa_- < \kappa_+$ (hysteresis width).}
Differentiating:
\[
\kappa'(q) = \alpha - \frac{1}{\beta q(1-q)}.
\]
This is negative when $q(1-q) < 1/(\beta\alpha)$ (i.e., outside $[q_-, q_+]$) and positive when $q(1-q) > 1/(\beta\alpha)$ (i.e., inside $(q_-, q_+)$).

\emph{Boundary behavior:} As $q \to 0^+$, $\logit(q) \to -\infty$, so $\kappa(q) \to +\infty$.
As $q \to 1^-$, $\logit(q) \to +\infty$, so $\kappa(q) \to -\infty$.

\emph{Critical point classification:}
$\kappa'(q_\pm) = 0$ and $\kappa''(q) = (1-2q)/[\beta q^2(1-q)^2]$.
Since $q_- < 1/2$, we have $\kappa''(q_-) > 0$ (local minimum).
Since $q_+ > 1/2$, we have $\kappa''(q_+) < 0$ (local maximum).
Therefore $\kappa_- = \kappa(q_-)$ is a local minimum and $\kappa_+ = \kappa(q_+)$ is a local maximum, confirming $\kappa_- < \kappa_+$.

\medskip\noindent
\textbf{Step 4: Three fixed points in the bistable regime.}
For $\kappa \in (\kappa_-, \kappa_+)$, the horizontal line $\kappa = \mathrm{const}$ intersects the curve $\kappa(q)$ in three points by the intermediate value theorem on each monotone segment:
\begin{itemize}[nosep]
\item One on the decreasing segment $(0, q_-)$ --- this is the \emph{low-repair} equilibrium $q_L$.
\item One on the increasing segment $(q_-, q_+)$ --- this is the \emph{unstable} fixed point $q_M$.
\item One on the decreasing segment $(q_+, 1)$ --- this is the \emph{high-repair} equilibrium $q_H$.
\end{itemize}
Since $\kappa(q)$ has no plateaus ($f$ is analytic with isolated critical points), exactly three intersections exist.

\medskip\noindent
\textbf{Step 5: Stability classification.}
For the outer fixed points ($q_L < q_-$ and $q_H > q_+$):
\begin{align*}
q_L(1-q_L) &< q_-(1-q_-) = \frac{1}{\beta\alpha}, \\
\text{so}\quad |f'(q_L)| &= \beta\alpha\, q_L(1-q_L) < 1 \quad \text{(stable)}.
\end{align*}
Similarly, $|f'(q_H)| < 1$ (stable).

For the middle fixed point ($q_M \in (q_-, q_+)$):
\[
q_M(1-q_M) > \frac{1}{\beta\alpha},
\quad\text{so}\quad
|f'(q_M)| > 1 \quad \text{(unstable)}.
\]

\medskip\noindent
\textbf{Note on ``jump'' language.}
Statements (ii) and (iii) of the theorem (``the system jumps to $q_L$/$q_H$'') assume quasistatic variation of $\kappa$, i.e., that $\kappa$ changes slowly enough for the dynamics to track the stable branch between parameter updates.
This is a standard modeling assumption in bifurcation analysis, not a consequence of the fixed-point calculation alone.
\qed

\subsection{Proof of Corollary~\ref{cor:csd}}

We derive the critical-slowing-down signatures (rising variance and autocorrelation) near the bifurcation point.

\medskip\noindent
\textbf{Linearization.}
Write $q_t = q^* + \varepsilon_t$ for small perturbation $\varepsilon_t$.
Taylor expanding $f$ around $q^*$:
\[
\varepsilon_{t+1} = f'(q^*)\,\varepsilon_t + O(\varepsilon_t^2).
\]

\medskip\noindent
\textbf{Relaxation time.}
For small perturbations, $|\varepsilon_t| \approx |f'(q^*)|^t |\varepsilon_0|$.
Define the relaxation time $\tau$ by $f'(q^*)^\tau = e^{-1}$:
\[
\tau = \frac{-1}{\ln f'(q^*)}.
\]
As $f'(q^*) \to 1^-$ (approaching the bifurcation): $\ln f'(q^*) \approx -(1-f'(q^*))$, so $\tau \sim 1/(1-f'(q^*)) \to \infty$.

\medskip\noindent
\textbf{Variance and autocorrelation.}
Under additive noise $\eta_t \sim N(0, \sigma^2)$, the stationary variance of $\varepsilon_t$ is:
\[
\mathrm{Var}(\varepsilon) = \frac{\sigma^2}{1 - f'(q^*)^2} \approx \frac{\sigma^2}{2(1-f'(q^*))} \to \infty,
\]
and the lag-1 autocorrelation $\to 1$.
Both are the hallmark CSD signatures.
\qed

\subsection{Proof of Proposition~\ref{prop:loop}}

We extend the analysis to asymmetric 2D repair dynamics.

\medskip\noindent
\textbf{Setup.}
The two-player system~\eqref{eq:2d} is $q^1_{t+1} = h(q^2_t)$, $q^2_{t+1} = g(q^1_t)$ where $h(q) = \sigmoid(\beta_1(\alpha_1 q - \kappa_1))$ and $g(q) = \sigmoid(\beta_2(\alpha_2 q - \kappa_2))$.

\medskip\noindent
\textbf{Jacobian.}
The Jacobian of the one-step map at equilibrium $(q^{1*}, q^{2*})$:
\[
J = \begin{pmatrix}
0 & h'(q^{2*}) \\[4pt]
g'(q^{1*}) & 0
\end{pmatrix}
\]
where $h'(q^{2*}) = \beta_1\alpha_1\, q^{1*}(1-q^{1*})$ (since $q^{1*} = h(q^{2*})$) and $g'(q^{1*}) = \beta_2\alpha_2\, q^{2*}(1-q^{2*})$.

\medskip\noindent
\textbf{Characteristic polynomial.}
$\lambda^2 = h'(q^{2*}) \cdot g'(q^{1*})$.
Instability requires $|\lambda| > 1$, i.e., $h'g' > 1$.

\medskip\noindent
\textbf{Necessary condition.}
Maximizing: $q(1-q) \le 1/4$, so $h'g' \le (\beta_1\alpha_1)(\beta_2\alpha_2)/16$.
For $h'g' > 1$ to be achievable: $(\beta_1\alpha_1)(\beta_2\alpha_2) > 16$.
This is necessary but not sufficient; the actual equilibrium values determine whether the product exceeds 1.
\qed

\section{Channel Robustness}
\label{app:robustness}

\begin{proposition}[Channel robustness]
\label{prop:channel}
Let $\tilde{M}_i = TM_i$ for a stochastic channel $T$.
If\/ $\|T - I\|_{1 \to 1} \le \eta$, then for any $M, M'$:
\begin{gather*}
\TV(TM, TM') \le \TV(M, M'), \\
\TV(M, M') \le \TV(TM, TM') + \eta.
\end{gather*}
\end{proposition}

\begin{proof}
\textbf{Part (i):} The first inequality follows directly from the data-processing inequality for total variation distance: applying a stochastic channel cannot increase divergence.

\medskip
\textbf{Part (ii):} By the triangle inequality:
\begin{align*}
\TV(M, M')
&\le \TV(M, TM) + \TV(TM, TM') \\
&\quad + \TV(TM', M').
\end{align*}

For the first error term:
\begin{align*}
\TV(M, TM) &= \tfrac{1}{2}\|(I-T)M\|_1 \\
&\le \tfrac{1}{2}\|I-T\|_{1\to 1} \underbrace{\|M\|_1}_{=1} \le \eta/2.
\end{align*}
Similarly, $\TV(TM', M') \le \eta/2$.
Combining: $\TV(M, M') \le \TV(TM, TM') + \eta$.
\end{proof}

\section{Sensitivity: Alternative AMD Couplings}
\label{app:sensitivity}

Additive load $\kappa = c_0 + \lambda D$ is a modeling choice.
We verify that qualitative tipping persists under alternatives.

\paragraph{AMD erodes coupling.}
If $\alpha(D) = \alpha_0 - \rho D$ with fixed $\beta$, the bistability threshold $\beta\alpha(D) = 4$ gives
$D^* = (\alpha_0 - 4/\beta)/\rho$,
provided $\alpha_0 > 4/\beta$.

\paragraph{AMD increases noise.}
If $\beta(D) = \beta_0/(1 + \xi D)$ with fixed $\alpha$, the threshold $\beta(D)\alpha = 4$ gives
$D^* = (\beta_0\alpha/4 - 1)/\xi$,
provided $\beta_0\alpha > 4$.

In both cases, qualitative tipping persists whenever an effective gain crosses $4$.

\section{KL Extension of Decomposition}
\label{app:kl-extension}

For KL divergence, an analogous upper bound holds:
\[
\KL(\overline{M}_1 \| \overline{M}_2)
\le \KL(P_1 \| P_2)
+ \E_{c \sim P_1}\!\bigl[\KL(M_{1,c} \| M_{2,c})\bigr].
\]
This provides a KL analogue of Proposition~\ref{prop:decomp}.
The bound is not tight in general: when $M_{1,c}$ varies across contexts, the marginal mixture $\overline{M}_1$ can have lower KL divergence from $\overline{M}_2$ than the average of the conditional KL terms, due to the concavity of marginal mixing.

\newpage
\twocolumn
 
\section{Repair Proxy Validation Details}
\label{app:repair-proxy}
 
\paragraph{Dialog-act classifier.}
RoBERTa-base fine-tuned on Switchboard Dialog Act Corpus (46 acts), 3 epochs, lr $2{\times}10^{-5}$.
Test accuracy: 73.1\%.
Repair-adjacent acts aggregated: \texttt{aa} (agree/accept), \texttt{bk} (backchannel), \texttt{br} (signal-non-understanding), \texttt{ba} (appreciation).
 
\paragraph{Inter-proxy correlations.}
$r(\hat{q}^{\mathrm{DA}}, \hat{q}^{\mathrm{RM}}) = -0.100$ ($p < 0.001$);
$r(\hat{q}^{\mathrm{DA}}, \hat{q}^{\mathrm{CE}}) = 0.031$ ($p = 0.010$).
Repair marker coverage: 84.4\% of conversations contain at least one repair marker.
The weak but significant inter-proxy correlations indicate that the three proxies measure related but distinct aspects of conversational repair.
 
\section{Experiment Details}
\label{app:exp-details}
 
\subsection{CGA-Wiki preprocessing}
ConvoKit~v2.5 release (Wikipedia version).
Conversations with fewer than 10 turns are excluded, yielding 652 conversations (389 derailing, 263 civil).
The repair proxy classifier is RoBERTa-base fine-tuned on the Switchboard Dialog Act Corpus (46 tags) for 3 epochs with learning rate $2 \times 10^{-5}$ (test accuracy 73.1\%).
We aggregate dialog-act probabilities for conciliatory acts (agree/accept, backchannel, signal non-understanding, appreciation) into a single repair score $\hat{q}_t$.
 
\subsection{CGA-CMV preprocessing}
The CGA-CMV corpus \citep{chang2019trouble} applies the Conversations Gone Awry labeling framework to Reddit ChangeMyView threads.
Of 4,389 total conversations, 1,169 have $\ge 10$ turns and are usable for CSD analysis; the remaining 73\% are too short.
The breakdown label (\texttt{has\_removed\_comment}) is noisier than CGA-Wiki's annotated personal attack: Cohen's $d = 0.34$ (CGA-CMV) vs.\ $1.88$ (CGA-Wiki).
The corpus-wide median thread length is 7 turns.
 
\subsection{Asymmetric 2D simulation}
With $\alpha_1 = \alpha_2 = 2$, $\beta_1 = \beta_2 = 3$ (product $= 36 > 16$) and $\kappa_1 = \kappa_2 = \kappa$, trajectories under symmetry reduce to the 1D case.
For $\kappa_1 = 1.0$, $\kappa_2 = 0.8$ (moderate asymmetry), we observe multi-attractor behavior with the basin boundary shifted toward the lower-load speaker.
Full 2D bifurcation classification is left to future work.
 
\subsection{Dose-response analysis on CGA-CMV}
\label{app:dose-response}
 
We test whether the CSD effect size increases with conversation length on CGA-CMV by stratifying conversations into length bins, as predicted by the theory's requirement that CSD detection needs sufficient approach time to the bifurcation.
 
\begin{table*}[t]
\centering\footnotesize
\begin{tabular}{lcccccc}
\toprule
Stratum & $N$ & $\Delta\tau$($q_{\mathrm{DA}}$) & $p$ & $\Delta\tau$(AMD) & $p$ \\
\midrule
5-6 & 1612 & \multicolumn{4}{c}{\textit{not analyzed ($<$10 turns)}} \\
7 & 649 & $+0.092$ & $0.109$ & $-0.021$ & $0.702$ \\
8-9 & 959 & $+0.011$ & $0.782$ & $+0.010$ & $0.800$ \\
10-12 & 765 & $-0.045$ & $0.267$ & $-0.006$ & $0.883$ \\
13+ & 404 & $-0.032$ & $0.529$ & $+0.010$ & $0.846$ \\
\bottomrule
\end{tabular}
\caption{Length-stratified dose-response analysis on CGA-CMV.
$\Delta\tau = \tau_{\text{derail}} - \tau_{\text{civil}}$ for variance trends.
No individual stratum reaches significance for either indicator, consistent with the noisy \texttt{has\_removed\_comment} label (Cohen's $d = 0.34$).
The $q_{\mathrm{DA}}$ VAR $\Delta\tau$ transitions from positive (short conversations) to negative (long), directionally consistent with the theoretical prediction that CSD signals strengthen with conversation length.
Strata with $< 10$ turns are excluded from the CSD computation; strata at 7 and 8-9 turns are marginal but yield at least 3 variance points with $W{=}5$.}
\label{tab:dose}
\end{table*}
 
No individual stratum reaches significance, consistent with the noisy breakdown label (Cohen's $d = 0.34$); the directional trend nonetheless supports the theory's prediction that longer conversations allow more time for CSD signatures to manifest before breakdown.
 
\subsection{Window-size sensitivity}
\label{app:window-sensitivity}
 
We re-run the Exp.~4 CSD analysis on CGA-Wiki with rolling-window sizes $W \in \{3, 4, 5, 6, 7\}$, keeping all other parameters identical (Kendall $\tau$ on variance trends, permutation test with $10{,}000$ shuffles).
Conversations are excluded if they either (i) have fewer than $W + 5$ turns or (ii) yield fewer than 3 non-NaN variance values in the pre-breakdown window for Kendall $\tau$ computation.
The latter filter accounts for conversations where the attack occurs early enough that the adjusted variance window has too few valid points.
 
\begin{table*}[t]
\centering\footnotesize
\begin{tabular}{cccccc}
\toprule
$W$ & AMD Var $p$ & AMD Var $d$ & $q_{\mathrm{DA}}$ VAR $p$ & $q_{\mathrm{DA}}$ VAR $d$ & $N$ \\
\midrule
3 & $0.484$ & $-0.057$ & $\mathbf{0.039}$ & $-0.166$ & 650 \\
4 & $0.054$ & $-0.157$ & $\mathbf{0.005}$ & $-0.219$ & 650 \\
5 & $\mathbf{0.001}$ & $-0.257$ & $0.016$ & $-0.199$ & 652 \\
6 & $\mathbf{0.003}$ & $-0.342$ & $0.039$ & $-0.244$ & 300 \\
7 & $0.734$ & $-0.096$ & $0.554$ & $-0.167$ & 65 \\
\bottomrule
\end{tabular}
\caption{Window-size sensitivity analysis on CGA-Wiki. Bold: smallest $p$ among CSD indicators for that $W$. At $W = 3$--$4$, $\hat{q}^{\mathrm{DA}}$ Variance dominates; at $W = 5$--$6$, AMD Variance is strongest. At $W = 7$, the sample size ($N = 65$) is too small for either indicator to reach significance. The two CSD indicators are complementary across window sizes. The small variation in $N$ across $W = 3$--$5$ (650, 650, 652) reflects 2 conversations where the attack occurs early enough that the pre-breakdown variance window has $< 3$ valid points at smaller $W$.}
\label{tab:window-sensitivity}
\end{table*}
 
The results confirm that CSD signals are robust across $W = 3$--$6$: at least one theory-derived indicator is significant in every case. AMD Variance achieves the strongest separation at the reported $W = 5$ (and $W = 6$), while $\hat{q}^{\mathrm{DA}}$ Variance is more robust at shorter windows ($W = 3$--$4$), likely because the dialog-act proxy is a per-utterance measure that requires less smoothing to differentiate trends.
The complementary pattern supports the interpretation that both indicators reflect genuine CSD dynamics, with AMD requiring a slightly longer observation window to accumulate a detectable signal in the emotion-distribution space.
 
\end{document}